%% file: paper.tex
\documentclass{article}

\usepackage[preprint]{neurips_2026}

\usepackage[utf8]{inputenc} 
\usepackage[T1]{fontenc}    
\usepackage{hyperref}       
\usepackage{url}            
\usepackage{booktabs}       
\usepackage{amsfonts}       
\usepackage{nicefrac}       
\usepackage{microtype}      
\usepackage{graphicx}
\usepackage{amsmath}
\usepackage{algorithm}
\usepackage{amssymb}
\usepackage{algorithmic}
\usepackage{wrapfig}
\usepackage{booktabs}
\usepackage{rotating}
\usepackage{xcolor}
\usepackage{soul}
\usepackage{multirow}
\usepackage{mathtools}
\usepackage{amsthm}
\newtheorem{assumption}{Assumption}
\newtheorem{theorem}{Theorem}
\newtheorem{lemma}{Lemma}

\newtheorem{definition}{Definition}
\newtheorem{remark}{Remark}

\newcommand{\eg}{{\it e.g.}, }

\title{ART: Attention Run-time Termination for Efficient Large Language Model Decoding}

\author{
Chen Qiu$^{1}$ \quad Guozhong Li$^{1}$ \quad Cristian McGee$^{2}$ \quad
Aritra Dutta$^{2}$ \quad Panos Kalnis$^{1}$ \\
$^{1}$King Abdullah University of Science and Technology \\
$^{2}$University of Central Florida \\
\texttt{\{chen.qiu, guozhong.li, panos.kalnis\}@kaust.edu.sa} \\
\texttt{\{cristian.mcgee, aritra.dutta\}@ucf.edu}
}

\begin{document}

\maketitle
\input{sections/Abstract}
\input{sections/Introduction}
\input{sections/Related_Work}
\input{sections/Methodology}
\input{sections/Experiments}

\input{sections/Conclusion}

\subsubsection*{Acknowledgments}
For compute time, this research used IBEX and Shaheen III, managed by the Supercomputing Core Laboratory at King Abdullah University of Science and Technology (KAUST), Saudi Arabia.
Aritra Dutta is partially supported by the Florida Department of Health Grant, AWD00007072, and the National Science Foundation Grant, 2321986.

\bibliography{paper}
\bibliographystyle{plainnat}
\input{sections/appendix}
\end{document}

%% file: sections/Abstract.tex
\begin{abstract}
Long-context decoding in Large Language Models (LLMs) is constrained by the cost of accessing and processing the  Key-Value (KV) cache.
Despite evidence that attention outputs depend jointly on keys and values, most existing KV management methods rely on key-only pruning, since incorporating values incurs prohibitive overhead.
In this paper, we propose \textbf{Attention Run-time Termination (ART)}, 
a lightweight run-time mechanism that tracks accumulated attention outputs during kernel execution 
and terminates subsequent KV block accesses once further contributions become negligible. 
Rather than replacing KV selection, ART dynamically terminates redundant KV traversal on top of existing dense or sparse attention policies. We introduce a stability-based criterion that monitors both magnitude and directional changes of intermediate attention outputs and provides a theoretical characterization of the resulting truncation error.
Experiments on the LongBench and RULER Needle-in-a-Haystack tasks show that ART increases the generation throughput of existing KV-cache methods by up to 20\% without compromising the result quality. 
\end{abstract}

%% file: sections/Introduction.tex
\section{Introduction}
\label{sec:Introduction}

\begin{wrapfigure}{r}{0.48\linewidth}
  \centering
    \vspace{-3em}
  \includegraphics[width=\linewidth]{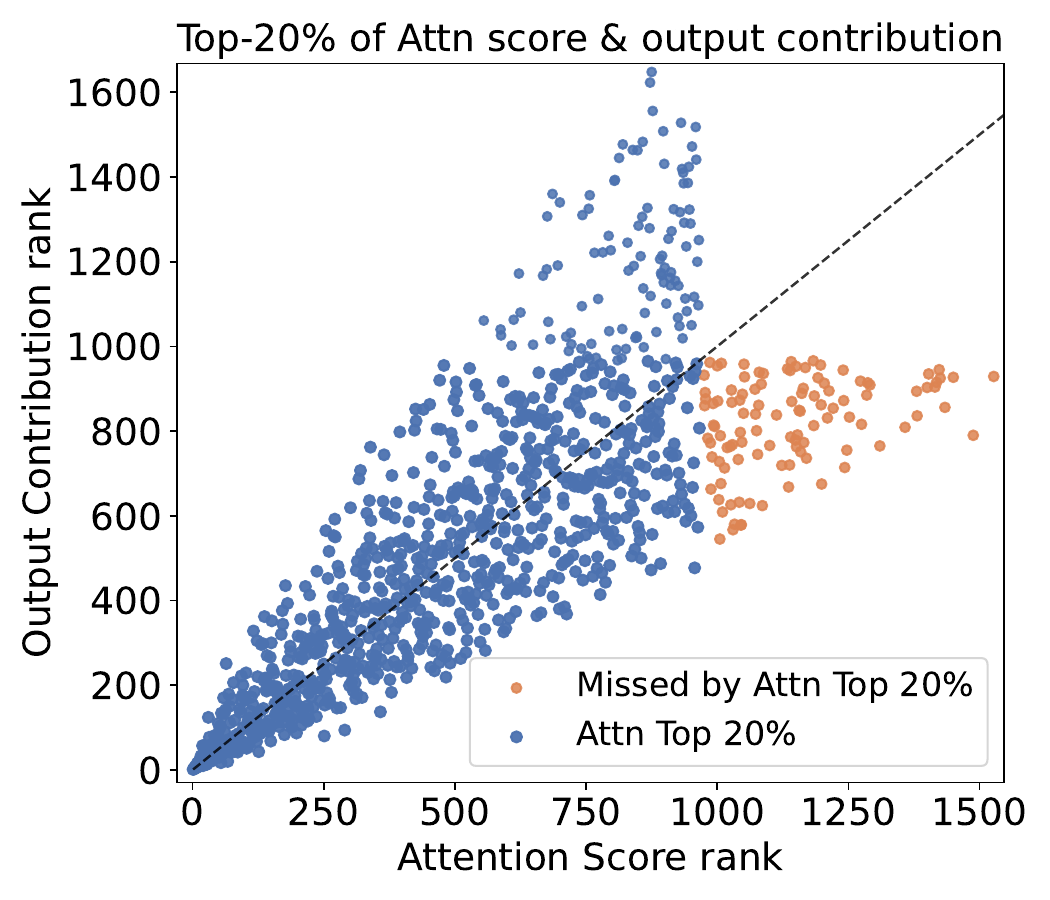}
  \vspace{-2em}
  \caption{Attention score versus output contribution. Orange points denote
high-contribution tokens missed by attention-score ranking. 
  }
  \label{fig:contr_vs_attn}
  \vspace{-1em}
\end{wrapfigure}

Large language models (LLMs)~\citep{touvron_llama_2023, yang_qwen3_2025} rely on a growing key--value (KV) cache to store the intermediate representations of previously generated tokens during autoregressive decoding. 
As the sequence length increases, each new query Q attends to an expanding set of cached keys K and values V, causing both latency and memory usage to grow linearly. Efficiently managing the KV cache, often referred to as KV pruning, has therefore become a central challenge in accelerating long-context inference.

Most existing approaches adopt key-centric KV cache pruning, retaining only tokens selected by heuristic or learned importance estimators.
Representative methods~\citep{xiao_efficient_2024,zhang_h_2o_2023,tang_quest_2024,liu_chunkkv_2025,liu2023scissorhands} typically base these estimators on query--key similarity or attention scores.
However, these proxies do not always reflect a token's true influence on the model output.
As illustrated in Figure~\ref{fig:contr_vs_attn}, this misalignment reveals a fundamental limitation of attention-based pruning: tokens with low attention weights can still have substantial influence through their value representations $\|A_{i,j} V_{j}\|_2$ .

Several recent studies~\citep{gu_obcache_2025,guo_attention_2024,akhauri_tokenbutler_2025} have attempted to incorporate value information into KV cache management, demonstrating the potential benefits of value-aware modeling.
However, these approaches typically rely on additional predictors, pre-computation, or offline analysis to estimate value contributions, introducing non-trivial overhead during inference.
This raises a natural question:
    \textbf{Can we capture the joint effect of keys and values at run-time with negligible additional cost?}


\begin{wrapfigure}{r}{0.6\linewidth}
  \centering
  \includegraphics[width=
  \linewidth]{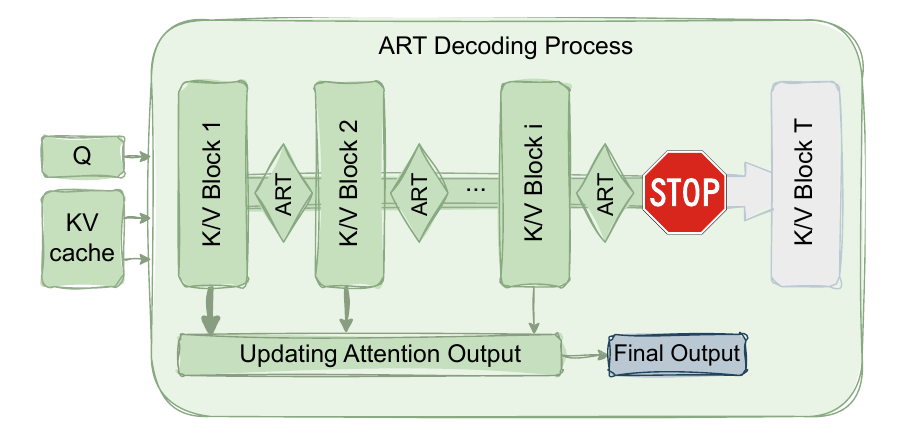}
  \vspace{-1em}
  \caption{Overview of the ART.
  As KV blocks are processed sequentially, ART monitors the intermediate attention output and triggers an early termination once the output stabilizes, skipping the remaining KV blocks during decoding.}
  \label{fig:ART}
\end{wrapfigure} 

Our key observation is that this joint effect is already reflected in the attention output itself. Rather
than estimating token importance before decoding, we can monitor how the accumulated attention
output evolves as additional KV blocks are processed. Modern FlashAttention-style kernels~\cite{dao_flashattention_2022,dao_flashattention-2_2023} naturally expose a sequence of intermediate attention outputs during block-wise execution.
As KV blocks are progressively loaded and processed (Figure~\ref{fig:ART}), each step incrementally updates the attention output, providing a direct signal of how much additional information the newly processed blocks contribute.
When these incremental updates become sufficiently small, even though not all KV blocks have been traversed, further computation and memory accesses would be unnecessary.
This enables progressive evaluation of the joint influence of keys and values at run-time, without relying on expensive pre-decoding estimation.

Building on this insight, we propose \textbf{Attention Run-time Termination (ART)}, a lightweight output-aware mechanism for terminating attention computation at run time. ART operates directly within block-wise attention kernels, incrementally monitoring the evolution of partial attention outputs as KV blocks are processed. Rather than replacing KV selection, ART dynamically terminates redundant KV traversal on top of existing dense or sparse attention policies. This makes ART a composable acceleration layer: a base KV policy determines which blocks are eligible to be processed, while ART decides at run time whether further traversal remains necessary.
Therefore, ART is orthogonal to most existing KV cache management methods and can be seamlessly integrated with them to refine cache utilization during decoding.

We evaluate ART on long-context benchmarks, including LongBench~\citep{bai2024longbench} and the Needle-in-a-Haystack (NIAH) tasks from RULER~\citep{hsieh2024ruler}. The results show that ART reduces decoding time per output token (TPOT) across all LongBench settings, with average reductions of 10.5\% on Mistral-7B and 9.7\% on Llama-3.1-70B, and improves large-batch generation throughput by up to 20\% when composed with existing KV cache management methods. On retrieval-sensitive NIAH tasks, ART is most effective when paired with retrieval-aware KV policies, highlighting its role as a dynamic termination layer rather than a replacement for KV selection.

Our contributions are summarized as follows:
\begin{itemize}
\item We propose ART, an output-aware run-time termination mechanism that acts as a composable acceleration layer over dense or sparse KV traversal policies.

\item We design a stability-based termination criterion that quantifies convergence in the output space by jointly monitoring scale and directional changes of intermediate attention outputs, thereby capturing the combined influence of keys and values.

\item We provide a theoretical characterization of ART's approximation behavior, deriving a truncation-error bound that relates early termination to output stability and residual attention contributions.

\item We demonstrate that ART improves generation throughput by up to 20\% at large
batch sizes and analyze its traversal-dependent speed--retrieval trade-off on
RULER NIAH.
\end{itemize}

%% file: sections/Related_Work.tex
\section{Related Work}
\label{sec:related}

In this section, we review prior work on efficient long-context inference.
We summarize KV cache management methods, and discuss efficient attention kernels and inference systems.



\paragraph{KV cache management.}
\label{sec:related:KV Cache Management}
Early approaches to efficient long-context inference adopt pattern-based pruning strategies, such as LM-Infinite~\citep{hanLMInfiniteZeroShotExtreme2024} and Attention Sink~\citep{xiao_efficient_2024}, which maintain a fixed-size active context by discarding older tokens or concentrating attention on a small set of anchor positions.
To move beyond static heuristics, subsequent methods estimate token importance using semantic, hierarchical, or dynamic criteria.
ChunkKV~\citep{liu_chunkkv_2025}, SnapKV~\citep{li_snapkv_2024}, Quest~\citep{tang_quest_2024}, and PyramidKV~\citep{cai_pyramidkv_2024} preserve coarse-grained global information while filtering redundant tokens through structured grouping.
H2O~\citep{zhang_h_2o_2023}, TOVA~\citep{orenTransformersAreMultiState2024}, and Scissorhands~\citep{liu2023scissorhands} further exploit the temporal persistence of importance scores to adaptively manage the KV cache in decoding steps.

Despite their effectiveness, these methods are fundamentally \emph{key-centric}, implicitly assuming that the attention matrix alone determines the output and overlooking the role of value (V) magnitudes.
Recent work~\cite{guo_attention_2024} challenges this assumption by showing that values (V) also encode critical semantic signals and significantly influence attention outcomes.
However, current value-aware approaches~\cite{gu_obcache_2025,akhauri_tokenbutler_2025} typically rely on auxiliary predictors or offline analysis, incurring substantial overhead and limiting their practicality for online serving.
This leaves an open gap to capture the joint influence of keys and values at \emph{run time} without additional latency.

\paragraph{Efficient attention kernels and inference systems.}
\label{sec:related:Efficient Attention Kernels & Inference Systems}
At the kernel level, FlashAttention~\cite{dao_flashattention_2022} and FlashAttention-2~\cite{dao_flashattention-2_2023} significantly improve attention efficiency through IO-aware tiling and online softmax.
By computing attention in blocks, these kernels reduce memory traffic between HBM and SRAM while incrementally accumulating attention outputs.
This block-wise execution model forms a key foundation for our method, as it exposes intermediate accumulation states that can be monitored at run time without modifying the attention formulation.

Building on such optimized kernels, serving frameworks like vLLM~\cite{kwon_efficient_2023} and SGLang~\cite{zheng_sglang_2024} further improve end-to-end efficiency via PagedAttention and advanced scheduling, addressing memory fragmentation and improving serving throughput.
However, these systems treat attention kernels as atomic operators.
Our work complements existing frameworks by opening this atomic operator and introducing a lightweight run-time early-termination mechanism directly within kernel execution, reducing unnecessary KV accesses at the source.


%% file: sections/Methodology.tex
\section{Methodology: ART}\label{sec:method}

In this section, we present \textbf{Attention Run-time Termination (ART)}.
We first examine execution properties of modern attention kernels, then introduce a stability-based run-time termination mechanism, and finally discuss ART's integration and correctness.


\begin{figure*}[t]
  \centering
  \includegraphics[width=1.05\textwidth]{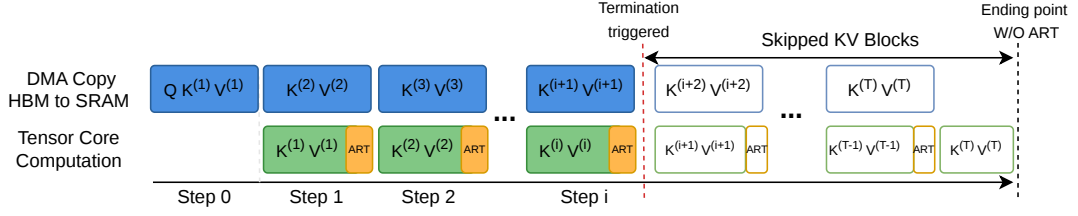} 
  \vspace{-2em}
  \caption{ART integrated into the FlashAttention execution pipeline.
FlashAttention overlaps DMA-based KV block prefetch from HBM with Tensor Core computation in a block-wise manner.
ART performs a lightweight run-time check on the evolving attention output during computation.
Once convergence is detected (\eg at Step~$i$), ART terminates the pipeline early by preventing further KV block prefetch and computation, reducing memory traffic and computation, with negligible impact on downstream generation quality.}\label{fig:ART with FA}
\vspace{-1em}
\end{figure*}

\subsection{Execution Properties of FlashAttention}
\label{sec:method:execution}
Although attention outputs are defined by a global softmax over all keys, modern FlashAttention-style kernels offer a crucial execution property: attention is computed in a streaming, tile-wise manner, where the output is incrementally accumulated while computation and memory transfers are overlapped (Figure~\ref{fig:ART with FA}).

An important implication is that the attention output may become sufficiently stable before all KV blocks are traversed, suggesting opportunities for early termination without materially affecting the final result.


To formalize this behavior, we consider the scaled dot-product attention for a query block
$Q \in \mathbb{R}^{N \times d}$ attending to keys
$K \in \mathbb{R}^{N \times d}$ and values
$V \in \mathbb{R}^{N \times d}$:
\begin{equation}\label{eq:attn}
O = \sigma\left(\frac{QK^\top}{\sqrt{d}}\right)V,
\end{equation}
where $\sigma$ is the softmax activation function acting row-wise; see definition in Appendix~\ref{appendix:theoretical}. FlashAttention computes Equation~\eqref{eq:attn} by partitioning keys and values into blocks $\{(K_t, V_t)\}_{t=1}^{T}$ and processing them sequentially.
Using a numerically stable streaming softmax formulation, the kernel maintains an internal accumulator that is updated after each block.
We denote by $O^{(t)}$ the attention output after processing the first $t$ blocks.
\begin{wrapfigure}{r}{0.43\linewidth}
  \centering
  \includegraphics[width=\linewidth]{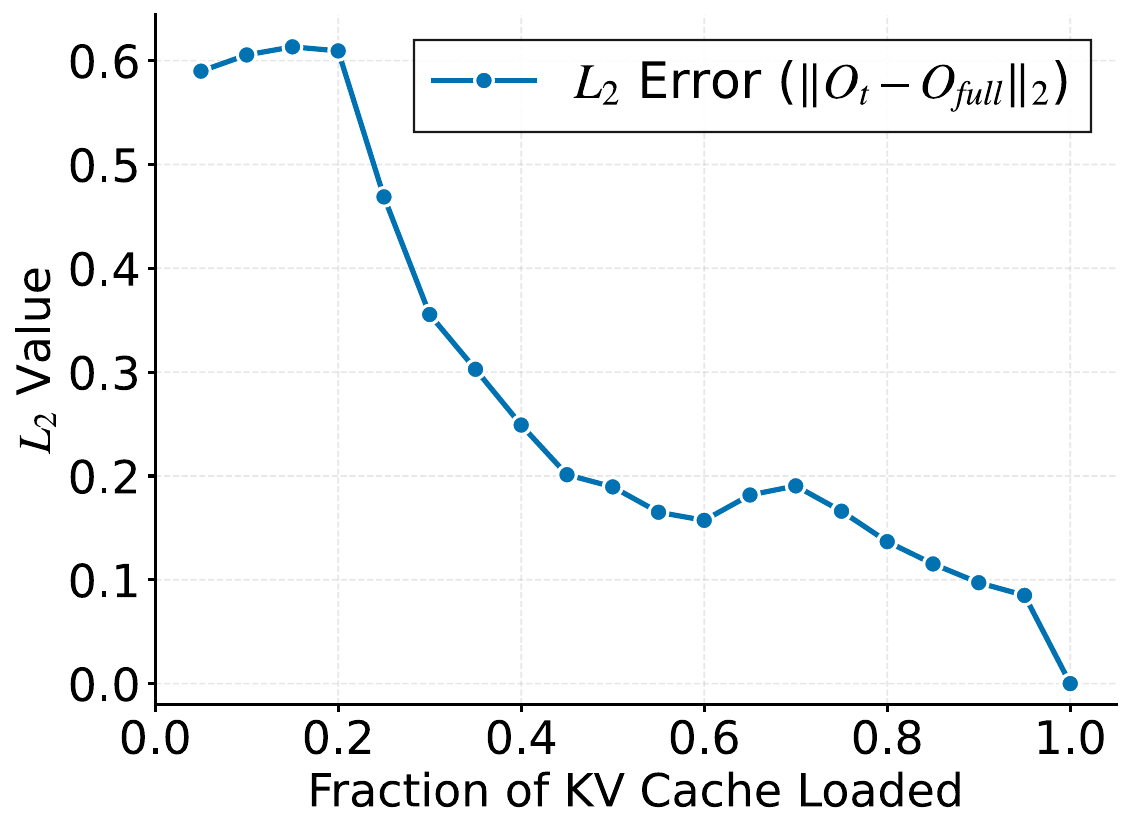} 
  \vspace{-1em}
  \caption{Convergence of attention output with recent KV retention. 
  The plot illustrates the relative $L_2$ error between the attention output from the full KV cache and that from a truncated recent-KV window, as a function of the loaded fraction.}
  \vspace{-4em}
  \label{fig:diff}
\end{wrapfigure}

As $t$ increases, $O^{(t)}$ converges to the final output $O^{(T)}$. 
As shown in Figure~\ref{fig:diff}, the attention output converges rapidly as KV blocks are processed.
Only a subset of blocks contributes substantial error reduction, while later blocks induce marginal changes.
This behavior suggests that attention outputs may stabilize well before all KV blocks are traversed.

\subsection{Stability-Based Run-time Termination}\label{sec:method:art}
Given the early stabilization behavior, a natural question is how to determine, at run time, whether the remaining KV blocks will meaningfully affect the final attention output.
A straightforward approach based on attention scores is insufficient, as high attention weights may be assigned to less informative or even zero-valued value vectors (Figure~\ref{fig:contr_vs_attn}).
This limitation suggests that key-based signals alone cannot reliably indicate when further attention computation becomes redundant.

Motivated by this insight, our idea is to track the evolution of the intermediate attention output.
Specifically, rather than estimating importance by attention score, we monitor how the partial attention output $O^{(t)}$ changes as additional KV blocks are incorporated.
If the output becomes sufficiently stable under a suitable traversal order, further traversal over KV blocks is less likely to alter the final result and can be considered for early termination.
By measuring stability directly in the output space, ART naturally captures the joint contribution of keys and values.


Directly computing full-vector norms or performing cross-thread reductions on the global accumulator would introduce non-trivial overhead and undermine the benefits of early termination.
This calls for a lightweight mechanism that can approximate output convergence without full-vector computation.


In optimized attention kernels utilizing NVIDIA Tensor Cores (\eg FlashAttention-2 \cite{dao_flashattention-2_2023}), the accumulator $O^{(t)}$ is distributed across threads via the swizzled Matrix Multiply-Accumulate (MMA) layout.
We construct a probe vector $x^{(t)} \in \mathbb{R}^{m}$ by collecting from all 32 lanes of warp 0, the leading element of their respective register fragments. Due to the interleaved nature of the MMA layout, these 32 elements constitute a deterministic, scattered subsample of the full head dimension rather than a contiguous slice, providing dispersed coverage across dimensions at negligible overhead.
We therefore monitor the stability of $x^{(t)}$ as a computationally efficient proxy signal for convergence. To improve decision reliability, the probe is integrated with the patience-based verification protocol (defined below) that dampens transient fluctuations before early termination is triggered. The effectiveness of the probe is studied in Appendix~\ref{app:probe_consistency}.

The iterative update of the attention accumulator induces a trajectory in a high-dimensional vector space. 
As successive KV blocks are aggregated, the output state $O^{(t)}$ evolves from its previous state $O^{(t-1)}$ in two geometric dimensions:
\begin{itemize}
    \item \textbf{Magnitude.} The vector may grow or shrink in magnitude due to the accumulation.
    \item \textbf{Direction.} The vector may rotate, shifting its orientation to align with newly discovered semantic information in the current block.
\end{itemize}

To characterize convergence in both aspects, we decouple the stability evaluation into scale and direction components:
\begin{align}
d_{\mathrm{scale}}^{(t)} = \|x^{(t)} - x ^{(t-1)}\|_2,
\quad 
d_{\mathrm{direction}}^{(t)} = 1-\cos (x^{(t)}, x^{(t-1)}),
\label{eq:direction}
\end{align}

An attention update is considered \emph{stable} when both the magnitude change and the directional deviation fall below predefined thresholds: $d_{\mathrm{scale}}^{(t)} < \tau$ \text{and} \quad
$d_{\mathrm{direction}}^{(t)} < \phi.$

The necessity of this decoupled design arises from the practical limitations of using a single unified metric. 
Although the $\ell_2$-norm difference captures both magnitude and directional changes, no fixed threshold can reliably balance these two aspects in practice.
Thresholds that are sensitive enough to detect directional misalignment tend to be overly restrictive to benign scale variations, whereas relaxed thresholds that tolerate scale changes often fail to capture meaningful directional shifts.
Decoupling scale and direction therefore yields a more robust and interpretable criterion for detecting output convergence.

\begin{algorithm}[t]
  \caption{ART run-time termination for FlashAttention}
  \label{alg:ART}
  \begin{algorithmic}[1]
    \REQUIRE Stability thresholds $\tau$ (scale) and $\phi$ (direction), patience $p$
    \STATE $O^{(0)} \gets \mathbf{0}$ \hfill \COMMENT {Initialize attention accumulator}
    \STATE $c^{(0)} \gets 0$ \hfill \COMMENT{Consecutive stable steps counter}
    \FOR{$t = 1$ to $T$}
        \STATE Load next KV tile $(K_t, V_t)$
        \STATE Update streaming attention output $O^{(t)}$
        \STATE $x^{(t)} \gets \textsc{Probe}(O^{(t)})$
    
        \STATE Compute scale $d_{\text{scale}}^{(t)} \gets \|x^{(t)} - x^{(t-1)}\|_2$
        \STATE Compute direction $d_{\mathrm{direction}}^{(t)} \gets 1 - \cos(x^{(t)}, x^{(t-1)})$ 
        \IF{ $d_{\mathrm{scale}}^{(t)} < \tau$ {\bfseries and} $d_{\mathrm{direction}}^{(t)} < \phi$ }
            \STATE $c^{(t)} \leftarrow c^{(t-1)} + 1$
        \ELSE
            \STATE $c^{(t)} \leftarrow 0$
        \ENDIF
        \IF{$c^{(t)} \ge p$}
            \STATE {\bfseries break} \hfill \COMMENT{Termination triggered}
        \ENDIF
    \ENDFOR
    \STATE Finalize normalization and {\bfseries return} $O^{(t)}$
  \end{algorithmic}
\end{algorithm}


While the above criterion captures convergence at a single step, transient fluctuations across neighboring KV blocks may still lead to premature termination.
To mitigate the risk of premature termination caused by transient fluctuations, we incorporate a patience mechanism. 
Let $c^{(t)}$ denote the number of consecutive stable steps up to tile $t$:
\begin{equation}
c^{(t)} =
\begin{cases}
c^{(t-1)} + 1, & \text{if\;\;} d_{\mathrm{scale}}^{(t)} < \tau \;\;\text{and} \quad
d_{\mathrm{direction}}^{(t)} < \phi, \\
0, & \text{otherwise}.
\end{cases}
\end{equation}

Early termination is triggered once $c^{(t)} \ge p$, where $p$ is the patience parameter.
Under this protocol, ART is activated only when the stability criteria are satisfied across $p$ consecutive KV blocks, thereby filtering out transient fluctuations.
As shown in Section~\ref{sec:Ablation_study}, this mechanism substantially reduces false positives in termination decisions while incurring negligible additional latency.
Algorithm~\ref{alg:ART} summarizes the complete run-time termination procedure.

\input{sections/Theory}

\subsection{Composable Termination over KV Traversal}
ART is designed as a composable run-time termination layer over KV traversal, rather than a standalone replacement for KV selection. Existing KV cache management methods determine which tokens or blocks are retained before decoding. ART takes this retained set as input, optionally reorders its traversal, and decides at run time whether the remaining blocks still need to be processed. This separation is important: \emph{the base policy defines the candidate evidence set, while ART reduces redundant traversal after the accumulated attention output has stabilized.} 

The effectiveness of ART depends on the traversal order. If high-contribution blocks are visited early, output stabilization provides a meaningful signal for terminating the remaining traversal. This is also the condition captured by our theoretical analysis: the truncation bound applies under a mass-regulated traversal, where the normalized contribution of the residual blocks decays after the final stable window begins. 
\begin{itemize}
    \item For full KV caching, the selected set contains all KV blocks. ART traverses
blocks in reverse temporal order. This
ordering is motivated by the recency bias of LLMs, which makes recently generated
tokens more likely to dominate the current attention output.

    \item  For pattern-based methods, such as StreamingLLM~\citep{xiao_efficient_2024}, ART applies the same
recency-first traversal over the retained KV blocks while explicitly preserving
the sink blocks required by the base method. 

    \item For methods that assign explicit
importance scores to retained KV blocks, such as SnapKV~\citep{li_snapkv_2024},
ART utilizes an importance-first traversal by processing the selected blocks in
descending order of their assigned priority. This is implemented through an
auxiliary index mapping for indirect addressing and does not alter the base
method's cache selection logic.
\end{itemize}

The traversal-based integration renders ART orthogonal to existing KV pruning methods.
The base method determines the candidate KV blocks, while ART dynamically
determines how many of them need to be visited for the current query. If all
selected blocks are processed, traversal reordering preserves the attention
formulation up to standard floating-point non-associativity. When ART terminates
early, it provides an additional dynamic truncation step on top of the base
method, reducing unnecessary KV block accesses while retaining the original
selection policy.

\noindent
\textbf{Discussion.}
Taken together, these design choices fundamentally distinguish ART from prior sparsity-based techniques.
While existing methods such as SnapKV~\cite{li_snapkv_2024} and H2O~\cite{zhang_h_2o_2023} rely solely on key-based importance estimates, ART introduces an \emph{ex-post}, output-aware convergence check.
By monitoring both the magnitude and direction of attention output updates, explicitly incorporating value vectors, ART continues computation only when values meaningfully alter the representation.
As a result, ART functions not as a standalone pruning heuristic, but as a
value-informed dynamic termination mechanism that complements existing KV
strategies. It improves the efficiency--accuracy trade-off when the underlying
traversal policy exposes relevant evidence before termination.


%% file: sections/Theory.tex
\paragraph{Conditional truncation bound.}

When ART terminates at block $t^*<T$, it returns the intermediate output
$o^{(t^*)}$ instead of the full-attention output $o^{(T)}$. We characterize
this approximation under a mass-regulated traversal:
After the final stable window begins, the remaining blocks have a decaying
normalized attention contribution along the chosen traversal order. Under this
condition and a non-degenerate probe, if ART terminates after $p$ consecutive
stable steps with scale tolerance $\tau$, then:
\begin{equation}
\label{eq:main}
\|o^{(T)}-o^{(t^*)}\|\le\left(\frac{\tau b+\Delta\mu}{\nu}\right)\ln\left(\frac{T}{t^*}\right),
\end{equation}
where $b=t^*-p+1$ is the beginning of the patience window and $\nu$ measures how well the probe preserves the relevant output-update directions. This result \emph{does
not guarantee task-level retrieval correctness for arbitrary traversal orders}; it formalizes the approximation error once the traversal has exposed the
dominant contributors and the residual mass is decaying (see Appendix~\ref{appendix:theoretical} for details).


%% file: sections/Experiments.tex
\section{Experimental Evaluation}
Our experiments show that ART improves decoding efficiency on LongBench and decreases time per output token (TPOT) when composed with existing KV cache management methods. We further use RULER NIAH to analyze its retrieval-sensitive failure mode under naive recency-first traversal.

\subsection{Experimental Setup}
\label{sec:Experimental setup}
Our experiments are conducted with Mistral-7B-Instruct-v0.3~\cite{jiangMistral7B2023}, Llama-3.1-70B-Instruct~\cite{meta2024llama31} and Qwen3-8B~\cite{yang_qwen3_2025} to evaluate long-context reasoning and generation performance. 
Experiments in this section are conducted on 1 or 4 NVIDIA A100-SXM4 GPUs with 80GB memory and NVIDIA H200 SXM GPU with 141GB memory.

\noindent
{\bf Benchmarked methods.}
We compare full KV (Baseline) with three representative KV-cache optimization methods: StreamingLLM~\cite{xiao_efficient_2024}, SnapKV~\cite{li_snapkv_2024}, PyramidKV~\cite{cai_pyramidkv_2024}. 
All methods use identical decoding configurations and random seed to ensure fair comparison; the detailed configurations for all comparative methods are listed in Appendix~\ref{sec:Configuration Detail}. We combine each of the above methods with ART (shown as \textbf{+ART}), using default parameters: $\tau = 10^{-5}$, $\phi = 10^{-3}$, and $p = 5$. Refer to Appendix~\ref{sec:Parameter Sensitivity} for a sensitivity analysis of the parameters. 

\noindent
\textbf{Metrics and measurement.}
\label{sec:Metrics and Measurement}
We focus on the decoding stage, where the KV cache is the main contributor to both latency and memory traffic. To assess generation quality, we report the LongBench score~\citep{bai2024longbench}. We further report accuracy on the Needle-in-a-Haystack (NIAH) tasks from RULER~\citep{hsieh2024ruler} to evaluate robustness in long-range retrieval scenarios, where premature termination could skip distant critical information.
For latency, we report Time Per Output Token (TPOT) as the primary metric.
To evaluate scalability under different batching configurations, we additionally report generation throughput, measured as output tokens per second.
Finally, to isolate the efficiency gains of ART from system-level overheads, we measure the runtime of the FlashAttention (FA) kernel itself, defined as the GPU execution time of FA kernel invocations recorded via CUDA events and averaged over decoding steps; refer to Appendix~\ref{sec:computational_overhead}.

\subsection{Overall Performance on LongBench}
\label{sec:Overall Performance on LongBench}

\begin{table*}[t]
\centering
\caption{Average LongBench score and decoding TPOT speedup (TPOT$\times$) on
Mistral-7B-Instruct-v0.3 and Llama-3.1-70B-Instruct. For each base
method, we report the original result and the result after integrating
ART. }
\label{tab:longbench_avg_main}
\small
\setlength{\tabcolsep}{4pt}
\begin{tabular}{@{}lrrrr@{\hskip 14pt}rrrr@{}}
\toprule
& \multicolumn{4}{c}{\textbf{Mistral-7B-Instruct-v0.3}}
& \multicolumn{4}{c}{\textbf{Llama-3.1-70B-Instruct}} \\
\cmidrule(lr){2-5} \cmidrule(lr){6-9}
& \multicolumn{2}{c}{LB Score}
& \multicolumn{2}{c}{TPOT$\times$}
& \multicolumn{2}{c}{LB Score}
& \multicolumn{2}{c}{TPOT$\times$} \\
\textbf{Method}
& Base & +ART & Base & +ART
& Base & +ART & Base & +ART \\
\midrule
Baseline       & 46.29 & 45.74 & 1.00 & \textbf{1.16}
               & 48.25 & 47.48 & 1.00 & \textbf{1.15} \\
\addlinespace
StrLLM(0.8)    & 40.29 & 40.20 & 1.08 & \textbf{1.24}
               & 42.31 & 41.80 & 1.13 & \textbf{1.27} \\
StrLLM(0.2)    & 31.11 & 31.09 & 1.18 & \textbf{1.30}
               & 33.88 & 33.51 & 1.52 & \textbf{1.66} \\
\addlinespace
SnapKV(0.8)    & 43.92 & 43.12 & 1.07 & \textbf{1.17}
               & 46.49 & 45.54 & 1.08 & \textbf{1.18} \\
SnapKV(0.2)    & 31.17 & 30.80 & 1.17 & \textbf{1.27}
               & 34.24 & 33.59 & 1.47 & \textbf{1.58} \\
\addlinespace
PyramidKV(0.8) & 42.57 & 42.00 & 1.06 & \textbf{1.19}
               & 44.10 & 43.31 & 1.12 & \textbf{1.25} \\
PyramidKV(0.2) & 30.07 & 29.97 & 1.17 & \textbf{1.29}
               & 35.93 & 35.47 & 1.47 & \textbf{1.60} \\
\bottomrule
\end{tabular}
\end{table*}

Table~\ref{tab:longbench_avg_main} reports the average LongBench scores and
decoding TPOT on Mistral-7B-Instruct-v0.3 and Llama-3.1-70B-Instruct. Across
both models, ART introduces only marginal quality score drop. 
In terms of decoding speed, ART consistently decreases TPOT across both dense and sparse KV-cache settings. For Mistral-7B, ART 
achieves an average decrease in TPOT of approximately $10.5\%$, whereas for Llama-3.1-70B the average TPOT reduction is around $9.7\%$, across all KV configurations. The results demonstrate that ART can work with a variety of existing KV management methods to achieve significant speed gains, with minimal effect on the quality of the generated output. Refer to Appendix~\ref{app:longbench_breakdown} for a full category-level breakdown of the results.




\begin{figure*}[t]
  \centering
  \includegraphics[width=.8\linewidth]{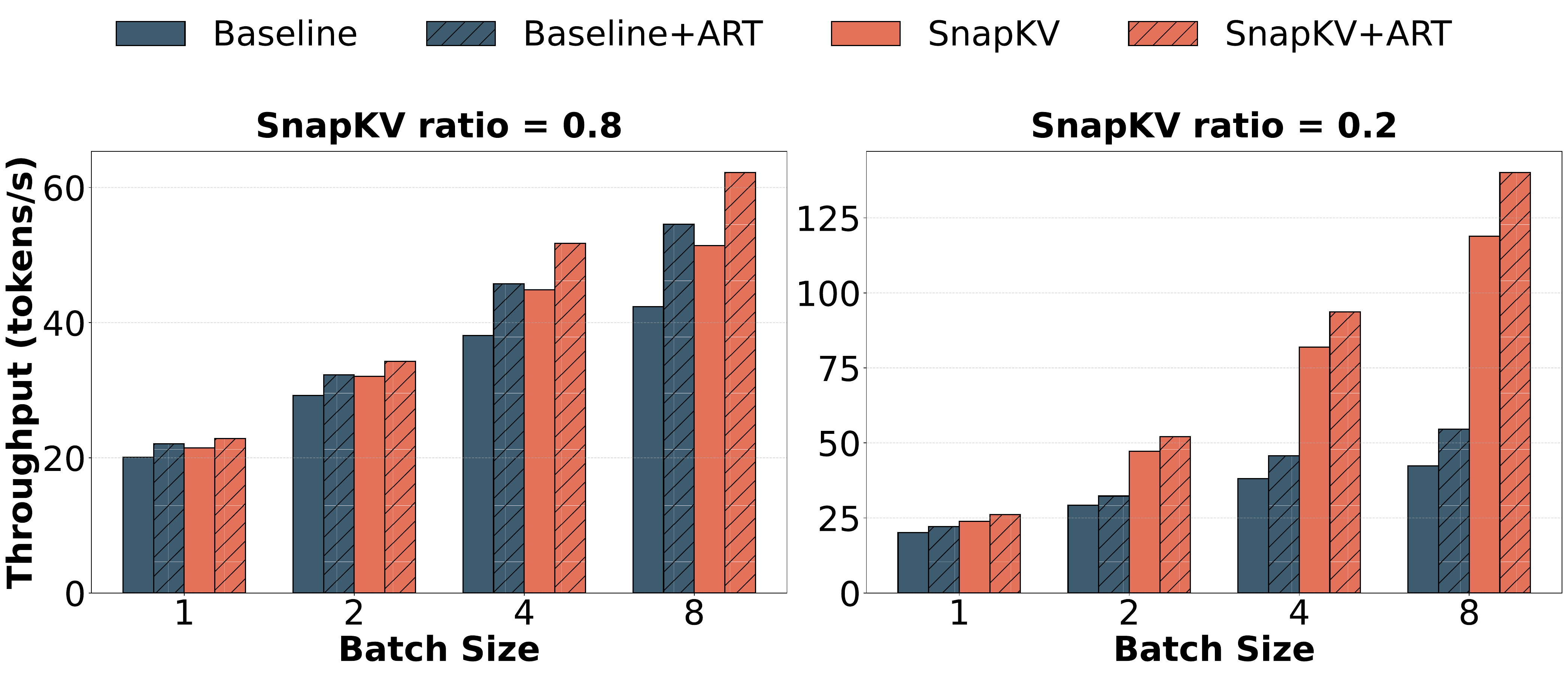} 
  \caption{Generation throughput (tokens/s) versus batch size for Full-KV (Baseline) and SnapKV, combined with ART. The left and right panels correspond to KV cache retention ratios of $0.8$ and $0.2$, respectively. SnapKV+ART performs consistently better, especially for large batch size.}
  \label{fig:bsz}
\end{figure*}

To evaluate ART under realistic serving scenarios, we measure generation throughput across batch sizes $B \in \{1, 2, 4, 8\}$.
Figure~\ref{fig:bsz} summarizes the results using SnapKV as a representative baseline.
Across all batch sizes and KV retention ratios, integrating ART consistently improves throughput.
Notably, the throughput gains become more pronounced as batch size increases, due to 
the memory-bound nature of long-context attention decoding.
As batch size grows, concurrent sequences amplify contention for HBM bandwidth due to repeated KV block accesses.
ART mitigates this bottleneck by reducing the number of accessed KV blocks per sequence.
Crucially, the overhead of ART’s run-time stability detection is constant per KV block and does not scale with batch size, whereas the savings from skipped memory accesses and computation scale linearly with the number of concurrent sequences.
This asymmetry leads to increasing net throughput gains at larger batch sizes.

The speedup of ART increases for long contexts; refer to Appendix~\ref{exp:context_length} for details.

\subsection{RULER Needle-in-a-Haystack Evaluation}

We evaluate ART on the Needle-in-a-Haystack (NIAH) tasks from RULER~\citep{hsieh2024ruler} using Qwen3-8B. NIAH directly tests long-range retrieval robustness, making it a challenging setting for ART: if termination is triggered before a distant evidence block is traversed, accuracy can degrade. 

Table~\ref{tab:niah} shows that applying ART directly to the full KV baseline improves speed but can noticeably reduce NIAH accuracy. This confirms that na\"ive recency-first traversal is not sufficient for retrieval-sensitive contexts, where critical evidence may appear far from the recent window. When ART is combined with retrieval-aware policies such as SnapKV and PyramidKV, it achieves speedup with a much smaller accuracy drop. This suggests that ART is most effective when the traversal policy exposes relevant evidence early, allowing the termination rule to skip redundant remaining blocks rather than replace the underlying retrieval mechanism. The StreamingLLM results further support this interpretation: since its retained context is dominated by sink tokens and recent tokens, its retrieval robustness is limited by the base policy itself, with ART providing the acceleration. 

Overall, the NIAH results clarify that ART is best viewed as an orthogonal run-time termination layer rather than a standalone retrieval method: its robustness depends on the traversal policy, while its efficiency gains come from dynamically skipping redundant KV blocks.


\begin{table}[t]
\centering
\caption{Accuracy (\%) and decoding TPOT speedup (TPOT$\times$) of Qwen3-8B on RULER NIAH,
For each base method, we report the original result and the result after
integrating ART.}
\label{tab:niah}
\small
\setlength{\tabcolsep}{4pt}
\renewcommand{\arraystretch}{1.08}
\begin{tabular}{@{}lcccccccc@{}}
\toprule
\multirow{2}{*}{\textbf{Method}}
& \multicolumn{2}{c}{\textbf{4K}}
& \multicolumn{2}{c}{\textbf{8K}}
& \multicolumn{2}{c}{\textbf{16K}}
& \multicolumn{2}{c}{\textbf{32K}} \\
\cmidrule(lr){2-3} \cmidrule(lr){4-5}
\cmidrule(lr){6-7} \cmidrule(lr){8-9}
& Acc. & TPOT$\times$
& Acc. & TPOT$\times$
& Acc. & TPOT$\times$
& Acc. & TPOT$\times$ \\
\midrule
Baseline       & 99.95 & 1.00          & 99.88 & 1.00          & 99.73 & 1.00          & 99.34 & 1.00 \\
\quad + ART    & 89.11 & \textbf{1.10} & 95.24 & \textbf{1.15} & 85.77 & \textbf{1.17} & 78.65 & \textbf{1.27} \\
\addlinespace
SnapKV         & 92.65 & 1.11          & 98.63 & 1.13          & 99.54 & 1.16          & 99.37 & 1.18 \\
\quad + ART    & 92.73 & 1.11          & 98.73 & \textbf{1.15} & 98.93 & \textbf{1.21} & 95.68 & \textbf{1.36} \\
\addlinespace
StreamingLLM   & 80.13 & 1.12          & 82.72 & 1.15          & 81.42 & 1.20          & 77.78 & 1.21 \\
\quad + ART    & 80.25 & 1.12          & 82.45 & \textbf{1.18} & 81.30 & \textbf{1.26} & 75.81 & \textbf{1.42} \\
\addlinespace
PyramidKV      & 90.09 & 1.10          & 97.34 & 1.12          & 99.36 & 1.18          & 99.24 & 1.19 \\
\quad + ART    & 90.72 & \textbf{1.11} & 96.75 & \textbf{1.15} & 97.78 & \textbf{1.21} & 96.00 & \textbf{1.37} \\
\bottomrule
\end{tabular}
\end{table}

\subsection{Ablation Study}
\label{sec:Ablation_study}

\begin{table*}[t]
\centering
\caption{Ablation study of ART on LongBench. We report the average score of Single-Document QA and Multi-Document QA (QA), summarization (Sum), few-shot learning (Fewshot) and overall LongBench score; the average FlashAttention kernel time per call; and their changes relative to the full ART configuration.}
\label{tab:dart_ablation}
\small
\begin{tabular}{l|ccccr|cr}
\toprule
\textbf{Method} &  QA & Sum & Fewshot&
\textbf{LB Score} & 
$\Delta$ \textbf{Score} & 
\textbf{Time (ms)} & 
$\Delta$ \textbf{Time} (\%) \\
\midrule
ART (full KV) & 35.58 & 26.40 & 63.04 &\textbf{45.74} & -- & 0.633 & -- \\
\midrule
w/o $d_{\mathrm{scale}}$     
& 14.78& 21.18& 27.74&19.34 & $-26.57$ & 0.135 & $-78.7$ \\
w/o $d_{\mathrm{direction}}$ 
& 35.32 & 26.25 & 62.64 &43.87 & $-2.04$  & 0.628 & $-0.8$ \\
w/o patience                 
& 33.58 & 25.96 & 58.11 &41.54 & $-4.37$  & 0.506 & $-20.1$ \\
\bottomrule
\end{tabular}
\end{table*}

Table~\ref{tab:dart_ablation} presents an ablation study that examines the contribution of the scale-based stability criterion,  the direction-based alignment measure, and the patience mechanism of ART.

Removing the scale-based stability criterion ($d_{\mathrm{scale}}$) leads to excessively early termination, resulting in a dramatically reduced FlashAttention kernel time and severe accuracy degradation, especially in QA.
Relying solely on directional consistency is insufficient to guarantee semantic convergence, since the attention output may stabilize in direction while still undergoing substantial magnitude changes.

The $d_{\mathrm{direction}}$ is introduced to address limitations of scale-only measurements.
As shown in Table~\ref{tab:dart_ablation}, incorporating $d_{\mathrm{direction}}$ yields a clear improvement in accuracy while incurring negligible additional kernel runtime.
This demonstrates that explicit directional alignment provides a low-cost yet effective refinement to convergence detection, improving robustness without compromising efficiency.

Removing the patience mechanism also degrades accuracy, despite reducing kernel execution time.
Patience plays an important role in filtering out transient fluctuations during early attention updates, preventing false-positive early termination decisions.
The slight increase in kernel time introduced by patience is therefore a necessary trade-off to ensure stable and reliable termination behavior.

Overall, these results confirm that the components of ART are complementary: 
$d_{\text{scale}}$ safeguards correctness, $d_{\text{direction}}$ enhances robustness, 
and the patience mechanism stabilizes termination decisions. 
Together, they strike an effective balance between decoding accuracy and kernel-level efficiency.


%% file: sections/Conclusion.tex
\section{Conclusion}\label{sec:conclusion}
We propose Attention Run-time Termination (ART), a lightweight run-time mechanism that accelerates long-context LLM decoding by terminating redundant KV traversal during attention execution. Unlike prior approaches that rely primarily on key-based importance estimates, ART monitors the convergence of intermediate attention outputs, thereby capturing the joint influence of keys and values at run time. ART opens a novel research direction: it introduces a new axis of optimization, output-space convergence, works at the kernel level, is complementary to existing KV cache management methods and introduces negligible overhead.

Our experiments show that ART provides significant  decoding-efficiency gains, in terms of TPOT and throughput, on top of dense and sparse KV traversal policies. On LongBench, ART introduces negligible effects on the quality of the results.
On retrieval-sensitive RULER NIAH tasks, quality depends on the base method.
The results suggest that ART should refine, rather than replace, retrieval-aware KV selection. In future work, we plan to explore adaptive thresholds and traversal policies across layers that further improve robustness under retrieval-sensitive contexts.




%% file: sections/appendix.tex
\newpage
\appendix



\section{Theoretical Error Bound Analysis}\label{appendix:theoretical}

Below, we prove the conditional truncation bound of Equation~\ref{eq:main}.

Let $K, V$ be partitioned into $T$ blocks $\{(K_t, V_t)\}_{t=1}^{T}$, with $K_t, V_t \in \mathbb{R}^{n \times D}$ and $nT=N$. Given a query block $Q \in \mathbb{R}^{N \times d},$ the current attention output at $t$ is given by 
$$
U^{(t)} = \sigma\left(\frac{QK_t^\top}{\sqrt{d}}\right)V_t,
$$
where $\sigma$ is the row-wise softmax operator defined element-wise on $S \in \mathbb{R}^{m\times l}$ by 
$$\sigma(S)_{i,j}=\frac{\exp(S_{i,j})}{\sum_{k=1}^l\exp(S_{i,k})}.$$ Given the $t$-th block and the $i$-th query row $q_i \in \mathbb{R}^{1 \times d}$, the vector's current output attention scale is given by
$\ell_{t,i}=\sum_{k=1}^n\exp(\frac{q_i(K^T_{t,k})}{\sqrt{d}})$, and the accumulating output attention scale is given by $A_{t,i}=\sum_{j=1}^t\ell_{j,i}.$

Then, the accumulating output for $q_i$ during block $t$ is obtained by the FlashAttention normalization procedure
$$o^{(t)}_i=\frac{A_{t-1,i}o_i^{(t-1)}+\ell_{t,i}u^{(t)}_i}{A_{t,i}},$$
where $u_i^{(t)}$ is the $i$-th row of $U^{(t)}$ and $o_i^{(t)}$ is the $i$-th row of the accumulating attention output $O^{(t)}.$ By initializing $o_i^{(1)}=u_i^{(1)}$, the normalization procedure defined above returns $O^{(T)}=O$, where $O$ is the entire output block as defined in Expression~\ref{eq:attn}. 

During autoregressive decoding, the most relevant query at each step is the most recently generated token $q_N$. The termination decision for $q_N$ is therefore the most critical; hence, our analysis focuses on a single query vector. 


For the remainder of this analysis we drop the subscript $i$ and let $o^{(t)}, u^{(t)}, A_t, \ell_t$ represent $o^{(t)}_i, u^{(t)}_i, A_{t, i}, \ell_{t,i}$ respectively. Given the FlashAttention normalization, the following recurrence relationship holds:
$$o^{(t)}-o^{(t-1)}=\frac{\ell_{t}}{A_{t}}(u^{(t)}-o^{(t-1)}).$$

This recurrence is the central theme of our analysis. The size of each increment is governed by the mass ratio $\frac{\ell_t}{A_t}$ and the direction $u^{(t)}-o^{(t-1)}.$ We consider a telescoping argument to represent the changes between the output obtained during ART's termination and the complete output.

Let $P \in \mathbb{R}^{m \times d}$ be a linear map. Let $p$ denote the patience, $\tau$ denote the tolerance scale, and $Po^{(t)}=x^{(t)}$ denote the probe vector as defined in \S\ref{sec:method:execution}. Let $t^*$ denote the stopping index for ART, and let $b=t^*-p+1$ represent the first iteration that satisfies the scale and direction conditions under which ART terminates.

\begin{definition}\label{def:update_dominance} \textbf{(Probe Update Dominance)}
   We define the probe update dominance as, $\Delta\coloneqq \max_{b\leq j \leq T}(\|P(u^{(j)}-o^{(j-1)})\|-\|P(u^{(b)}-o^{(b-1)})\|)$.
\end{definition}
\begin{remark}\label{remark:update_dominance}
    Above definition implies that for all $j \ge b$ the expression $\|P(u^{(j)}-o^{(j-1)})\| \le \|P(u^{(b)}-o^{(b-1)})\|+\Delta$ holds.
\end{remark}

\begin{definition} \label{def:terminate}
    \textbf{(ART Termination)} ART terminates on block $t^*$ after $p$ consecutive stable steps of $$\|x^{(j)}-x^{(j-1)}\| \le \tau,\quad~{\rm for~ all }~j={b, b+1, \cdots, t^*}. $$
\end{definition}

\begin{assumption}\label{assum:mass}
    \textbf{(Mass-Regulated Traversal)}
Let $\mu=\frac{b\ell_b}{A_b} $. There exists constants $\hat{\ell} > 0$ and $h > 0$ such that for all $j \ge b$:
{(i)} $\ell_j \le \hat{\ell}$,{(ii)} $A_j \ge h j$, with $\hat{\ell}/h = \mu$.
\end{assumption}

The mass-regulated traversal assumption formalizes the regime in which ART is designed to operate.
ART prioritizes KV blocks that are likely to carry large attention mass, either through recency-first traversal or through an external importance ordering.
After the first \(b\) blocks have been accumulated, later blocks are therefore expected to have bounded marginal mass relative to the accumulated mass.
Even when all blocks have comparable mass, \(\ell_j \approx \ell\) and \(A_j \approx j\ell\), yielding \(\ell_j/A_j \approx 1/j\).
Thus, the assumption captures the natural dilution of each newly visited block as the accumulated attention mass grows.

\begin{assumption} \label{assum:probe}
    \textbf{(Probe Constraints)} For some $\nu > 0$, the linear probe, $P \colon \mathbb{R}^d \to \mathbb{R}^m$ satisfies
\[
\nu \|z\| \le \|P z\| \le \|P\|_2 \|z\|,\quad {\rm for~all}
~z \in \text{Span}\{o^{(j)} - o^{(j-1)} : 1 \le j \le T\}. 
\]
\end{assumption}

The probe is introduced as an efficiency-motivated surrogate for evaluating output stability without explicitly materializing the full attention output during kernel execution.
In Appendix~\ref{app:probe_consistency}, we provide empirical evidence that stability measured in the probe space is well aligned with the oracle stability criterion computed from the full attention output. Now, we are set to quote our intermediate results. 

Lemma~\ref{lemma:harmonic_decay} follows directly from Assumption~\ref{assum:mass}. This assumption is not meant to hold for an arbitrary KV order. It states that once we reach the final stable window, the remaining blocks contribute decreasing normalized mass. This is natural in our setting because $\ell_j/A_j$ represents the relative weight $\ell_j$ contributes to the accumulated attention output factor $A_j$. Our empirical results indirectly support this interpretation, since Figure~\ref{fig:diff} demonstrates that attention output often stabilizes before the entire KV cache is processed. Table~\ref{tab:longbench_avg_main} shows that ART gains speed with little score loss, suggesting the skipped blocks often have small residual effect. Furthermore, Table~\ref{tab:niah} suggests that retrieval-aware traversals such as SnapKV and PyramidKV reinforce this behavior by exposing important blocks earlier, making the decay even more pronounced in practice.

\begin{lemma}
    \textbf{(Harmonic Accumulation Decay)} \label{lemma:harmonic_decay} Let Assumption~\ref{assum:mass} hold. Then for all $j \geq b$ we have $\frac{\ell_j}{A_j} \leq \frac{\mu}{j}.$
\end{lemma}

\begin{proof} Assumption~\ref{assum:mass} gives:
    $\frac{\ell_j}{A_j} {\leq} \frac{\hat{\ell}}{jh}=\frac{\mu}{j}$, and completes proof. 
\end{proof}

Lemma~\ref{lemma:harmonic_decay} tells us that the mass factor of $(o^{(j)}-o^{(j-1)})$ grows by the order $\mathcal{O}(1/j).$  Our next Lemma,
Lemma~\ref{lemma:post_termination}, provides an adaptive bound on the probe vectors based on the threshold, $\tau$, which is user-inferred and can be chosen as small as the user wants. Since ART terminates at iteration $t^*,$ then it is obvious that the previous $p$ steps satisfy the difference condition. 
\begin{lemma}\label{lemma:post_termination}
    \textbf{(Post-Termination Probe Decay)} Let Assumption~\ref{assum:mass} hold, then for all $k \geq 1$
    $$\|x^{(t^*+k)}-x^{(t^*+k-1)}\| \le \frac{\tau b}{t^*+k}<\tau.$$
\end{lemma}

\begin{proof}
    The linearity of $P$ gives $x^{(j)}-x^{(j-1)}=\frac{\ell_j}{A_j}P(u^{(j)}-o^{(j-1)}).$ 
    
    By Remark \ref{remark:update_dominance}, for every $j \ge b,$
    $$\|x^{(j)}-x^{(j-1)}\| =\frac{\ell_j}{A_{j}}\|P(u^{(j)}-o^{(j-1)})\| \le \frac{\ell_j}{A_j}(\|P(u^{(b)}-o^{(b-1)})\|+\Delta).$$
    Assumption~\ref{assum:mass} and Definition~\ref{def:terminate} give $\|x^{(b)}-x^{(b-1)}\| < \tau$, hence $\|P(u^{(b)}-o^{(b-1)})\| < \frac{A_b}{\ell_b} \tau{=} \frac{b}{\mu}\tau$. Then,
    $$\|x^{(j)}-x^{(j-1)}\| \le \frac{\ell_j}{A_j}(\|P(u^{(b)}-o^{(b-1)})\|+\Delta)  \le \frac{b}{j}\tau+\frac{\mu}{j}\Delta.$$ Setting $j=t^*+k$ gives the result $\|x^{(t^*+k)}-x^{(t^*+k-1)}\| \le \frac{b\tau + \mu \Delta}{t^*+k}<\tau$.
\end{proof}
Taken together with Assumption~\ref{assum:mass}, Lemma \ref{lemma:post_termination} shows that the difference continually shrinks by a factor of $\mathcal{O}(b/j).$ Our next Lemma quantifies the difference between the current output vector, $u^{(t)}$, and the previous accumulating attention output vector, $o^{(t-1)}.$


\begin{lemma}\label{lemma:uniform_bound} \textbf{(Uniform Difference Bound)} For $t=2,\cdots, T,$ $\|u^{(t)}-o^{(t-1)}\| \le 2C_v$, where $C_v$ is the largest row norm in the value matrix $V$, i.e., $C_v \coloneqq \max{_{1 \le k \le N}} \|v_{k}\|.$
\end{lemma}

\begin{proof}
    Since each $u^{(t)}$ represents the matrix product of a softmax output vector and the block $V_t,$ then each $u^{(t)}$ is a convex combination of the rows $\{(v_t)_r\}_{1 \leq r \leq n}.$ Hence,
    $$\|u^{(t)}\| \le \max_r\|(v_t)_r\| \le C_v.$$

    Each $o^{(t)}$ can be written as $\sum_{j=1}^t\frac{\ell_j}{A_t}u^{(j)},$ where each $\frac{\ell_j}{A_t} >0$ and sum to 1. Hence, $o^{(t)}$ is a convex combination of $u^{(t)}.$ Hence,
    $$\|o^{(t)}\|=\|\sum_{j=1}^t\frac{\ell_j}{A_t}u^{(j)}\| \le \max_{j}\|u^{(j)}\| \le C_v.$$

    By the triangle-inequality, $\|u^{(t)}-o^{(t-1)}\| \le \|u^{(t)}\|+\|o^{(t-1)}\| \le 2C_v,$ hence the result.
\end{proof}

Theorem~\ref{thm:probe_termination} is the central result that bounds the truncation error using the probe-based termination condition. The bound depends on the tolerance $\tau,$ the probe constant $\nu,$ and the iteration ratio $T/t^*.$ 

\begin{theorem}\label{thm:probe_termination} 
    Let Assumptions~\ref{assum:mass}-\ref{assum:probe} hold, then 
    $$\|o^{(T)}-o^{(t^*)}\| \le \frac{\tau b+\Delta\mu}{ \nu}\ln \left(\frac{T}{t^*}\right).$$
\end{theorem}

\begin{proof}
    The triangle-inequality and Assumption~\ref{assum:probe} gives
    $$\|o^{(T)}-o^{(t^*)}\| \le \sum_{j=t^*+1}^T \|o^{(j)}-o^{(j-1)}\| \le \frac{1}{\nu} \sum_{j=t^*+1}^T \|x^{(j)}-x^{(j-1)}\|.$$
    Using Lemma~\ref{lemma:post_termination}, the following result can be obtained 
    $$ \frac{1}{\nu} \sum_{j=t^*+1}^T \|x^{(j)}-x^{(j-1)}\| \le \frac{\tau b+\Delta \mu}{\nu}\sum_{j=t^*+1}^T\frac{1}{j} \le \frac{\tau b+\Delta \mu}{\nu}\int_{t^*}^T\frac{1}{x}dx = \frac{\tau b+\Delta \mu}{\nu} \ln \left( \frac{T}{t^*}\right).$$
\end{proof}

Theorem~\ref{thm:uniform_bound} provides a complementary bound through Lemma~\ref{lemma:uniform_bound}, expressed in terms of value matrix norm $C_v$ and the ratio $T/t^*.$

\begin{theorem}\label{thm:uniform_bound} Let Assumption~\ref{assum:mass} hold, then
$$\|o^{(T)}-o^{(t^*)}\| \le 2\mu C_v \ln\left(\frac{T}{t^*}\right).$$
    
\end{theorem}

\begin{proof}
    The triangle-inequality and Lemma~\ref{lemma:uniform_bound} gives
    $$\|o^{(T)}-o^{(t^*)}\| \le \sum_{j=t^*+1}^T\|o^{(j)}-o^{(j-1)}\|=\sum_{j=t^*+1}^T\frac{\ell_j}{A_j}\|o^{(j)}-o^{(j-1)}\| \le 2C_v\sum_{j=t^*+1}^T\frac{\mu}{j}.$$
    
    We finish the proof by using the same upper bound on the sum $\sum_{j=t^*+1}^T\frac{1}{j} \le \ln(\frac{T}{t^*}):$

    $$\|o^{(T)}-o^{(t^*)}\| \le 2\mu C_v\sum_{j=t^*+1}^T\frac{1}{j} \le 2 \mu C_v \int_{t}^T\frac{dx}{x}=2\mu C_v\ln \left( \frac{T}{t^*}\right).$$
\end{proof}

Theorems 1 and 2 can be put together under a single statement, as given in Remark \ref{corr:bound}. 
\begin{remark}\label{corr:bound}
    Let Assumptions~\ref{assum:mass}-~\ref{assum:probe} hold, then
$$\|o^{(T)}-o^{(t^*)}\| \le \min\{\frac{\tau b + \mu \Delta}{\nu},2\mu C_v\} \ln\left(\frac{T}{t^*}\right).$$
\end{remark}


\section{Probe--Oracle Consistency}
\label{app:probe_consistency}

ART applies a lightweight probe extracted from the distributed attention
accumulator rather than reconstructing the full attention output at every KV
block. While this design is essential for maintaining low kernel overhead, it
raises a natural question: whether the probe observes the same convergence
behavior as the full attention output. To validate this, we compare the
probe-based stopping decision against an oracle detector that applies the same
stability rule to the fully materialized attention output.

Concretely, for each decoding, we record the stopping block predicted
by the probe and compare it with the stopping block obtained by the oracle. We
classify the outcomes into three categories. An agreement means that the probe
and oracle make the same stopping decision, either stopping at the same block or
both deciding not to stop. A false positive (FP) means that the probe stops
earlier than the oracle, or that the probe stops while the oracle never stops.
This case corresponds to premature termination and is the only category that
may introduce accuracy risk. A false negative (FN) means that the oracle stops
earlier than the probe, or that the oracle stops while the probe never stops.
This case is conservative: it may reduce the achievable speedup but does not
increase approximation risk. For samples where both the probe and oracle stop
but at different blocks, we also report the block offset $\Delta_{\mathrm{blk}}$.

Table~\ref{tab:art_probe_consistency} reports the consistency results on
Qwen3-8B, which contains 36 layers with head dimension 128. We evaluate both
the effect of sequence length and the effect of layer depth. In the sequence
length study, we fix the last layer and place the needle at the middle of the
context. Across sequence lengths from 1K to 8K tokens, the probe almost exactly
matches the oracle decision. In particular, the agreement remains 100.00\% up
to 4K tokens and 99.88\% at 8K tokens, with only 0.13\% false positives. This
indicates that increasing the number of KV blocks does not noticeably degrade
the reliability of the probe.

We further evaluate different layers under a fixed 4K-token context. The probe
remains highly consistent with the oracle, although the agreement varies across
layers. The last layer again shows perfect agreement, while intermediate layers
exhibit slightly higher mismatch rates. In Layer 17, the false-positive rate is
5.56\%, and the mean block offset among mismatched stopping cases is only
$1.34\pm0.47$ blocks. This suggests that even when the probe stops earlier than
the oracle, the discrepancy is typically small. Since ART uses a patience-based
criterion, such small offsets can be further controlled by increasing the
patience parameter or using more conservative thresholds.

Overall, these results provide empirical support for the probe approximation
used by ART. The probe is not assumed to be an exact reconstruction of the full
attention output; instead, it is required to preserve the convergence behavior
relevant to the stopping decision. The high agreement and low false-positive
rates in Table~\ref{tab:art_probe_consistency} indicate that the probe satisfies
this requirement in practice, especially in the late layers where decoding
decisions are most directly reflected in the final representation.

\begin{table}[t]
\centering
\caption{%
    Consistency between the probe and the oracle on Qwen3-8B
    (36 layers, head\_dim\,=\,128).
    \textbf{FP}: probe stops before oracle, or probe stops while oracle never does
    (premature termination risk).
    \textbf{FN}: oracle stops before probe, or oracle stops while probe never does
    (conservative; no accuracy risk).
    $\boldsymbol{\Delta}$\textbf{blk}: mean $\pm$ std block offset among mismatched stopping cases (i.e., both probe and oracle stop, but at different blocks)..
}
\label{tab:art_probe_consistency}
\small
\begin{tabular}{llcccc}
\toprule
\textbf{Condition} & $n$ & \textbf{Agree\,$\uparrow$} & \textbf{FP\,$\downarrow$} & \textbf{FN\,$\downarrow$} & \boldmath$\Delta$\textbf{blk} \\
\midrule
\multicolumn{6}{l}{\textit{(a) Sequence length} \quad (layer 35, needle @ 50\%)} \\
\midrule
\phantom{xx}1\,024 tokens \; & 800 & \textbf{100.00\%} & \textbf{0.00\%} & 0.00\% & — \\
\phantom{xx}2\,048 tokens \; & 800 & \textbf{100.00\%} & \textbf{0.00\%} & 0.00\% & — \\
\phantom{xx}4\,096 tokens \; & 800 & \textbf{100.00\%} & \textbf{0.00\%} & 0.00\% & — \\
\phantom{xx}8\,192 tokens  & 800 & 99.88\%           & 0.13\%          & 0.00\% & — \\
\midrule
\multicolumn{6}{l}{\textit{(b) Layer depth} \quad (4\,096 tokens, needle @ 50\%)} \\
\midrule
\phantom{xx}Layer\;7   & 1\,600 & 99.50\%           & 0.50\%          & 0.00\% & 1.00$\pm$0.00 \\
\phantom{xx}Layer\;17       & 1\,600 & 94.06\%           & 5.56\%          & 0.38\% & 1.34$\pm$0.47 \\
\phantom{xx}Layer\;35       & 1\,600 & \textbf{100.00\%} & \textbf{0.00\%} & 0.00\% & — \\
\bottomrule
\end{tabular}
\end{table}

\section{Benchmark Dataset Description}
\label{app:longbench_description}

In this study, we utilize LongBench~\cite{bai2024longbench} and RULER~\cite{hsieh2024ruler} as our evaluation benchmarks. LongBench is the first bilingual, multi-task benchmark specifically designed for long-context Large Language Models (LLMs). This benchmark comprises 21 distinct datasets across 6 key task categories, aiming to provide a comprehensive assessment of a model's capabilities in understanding, generation, and reasoning over long texts.

Most tasks in LongBench are derived from existing public datasets but have been rigorously cleaned and filtered to ensure suitability for long-context evaluation (e.g., filtering out short text samples). The benchmark covers both English and Chinese languages. The average length of the English datasets is 6,711 words, while the Chinese datasets average 13,386 characters.

LongBench consists of the following six major task categories:

\begin{itemize}
    \item \textbf{Single-document QA}: Includes NarraQA, Qasper, MultiFieldQA-en, and MultiFieldQA-zh. These tasks require the model to retrieve and integrate information from a long document to answer specific questions.
    \item \textbf{Multi-document QA}: Includes HotpotQA, 2WikiMultihopQA, Musique and DuReader. These tasks involve complex reasoning and information synthesis across multiple documents.
    \item \textbf{Summarization}: Includes GovReport, QMSum, MultiNews and VCSUM. The goal is to generate high-quality summaries for long meeting transcripts, news collections, or government reports.
    \item \textbf{Few-shot Learning}: Includes TREC, LSHT, TriviaQA, and SAMSum. These tasks provide multiple long-context examples to evaluate the model's in-context learning capabilities.
    \item \textbf{Synthetic Tasks}: Includes Passage Retrieval (en and cn) and Number String. These are specifically designed pressure tests to detect the model's ability to accurately locate specific information within an extremely long context.
    \item \textbf{Code Completion}: Includes LCC and RepoBench-P. These tasks evaluate the model's understanding and completion abilities within long code repository files.
\end{itemize}

We additionally evaluate on the Needle-in-a-Haystack (NIAH) tasks from RULER~\cite{hsieh2024ruler}, a synthetic benchmark designed to stress-test a model's ability to retrieve specific information embedded within long, distracting contexts. RULER's NIAH suite comprises eight tasks of increasing complexity, covering four retrieval scenarios:

\begin{itemize}
    \item \textbf{Single-needle retrieval} (\texttt{niah\_single\_1/2/3}): A single key--value pair is hidden within a long distractor text. Variants differ in needle complexity: short word, long phrase, and UUID string, respectively.
    \item \textbf{Multi-key retrieval} (\texttt{niah\_multikey\_1/2/3}): Multiple needles with distinct keys are inserted; the model must retrieve the value for a specific queried key. Variants scale the number of distracting needles from 2 to 4.
    \item \textbf{Multi-value retrieval} (\texttt{niah\_multivalue}): A single key is associated with multiple values spread across the context; the model must retrieve all of them.
    \item \textbf{Multi-query retrieval} (\texttt{niah\_multiquery}): Multiple key--value pairs are hidden, and the model is asked to answer queries about several of them simultaneously.
\end{itemize}

We evaluate all eight NIAH tasks at context lengths of 4K, 8K, 16K, and 32K tokens. Each sample is scored as exact match (0 or 1) and results are averaged across samples and tasks.



\section{Configuration Detail}
\label{sec:Configuration Detail}

\begin{table}[h]
    \centering
    \caption{Hyperparameter configurations mapped to original paper notations.}
    \label{tab:kv_config}
    \small
    \begin{tabular}{l@{\hspace{0.5cm}}l@{\hspace{0.5cm}}c@{\hspace{0.5cm}}c}
        \toprule
        \textbf{Method} & \textbf{Paper Notation / Term} & \textbf{Symbol} & \textbf{Value} \\
        \midrule
        StreamingLLM & Initial Tokens (Attention Sinks) & - & 4 \\
        SnapKV & Observation Window & $L_{obs}$ & 32 \\
               & Pooling Kernel Size & - & 5 \\
        PyramidKV & Instruction Tokens (Local Window) & $\alpha$ & 128 \\
                  & Pooling Kernel Size & $\beta$ & 5 \\
        \bottomrule
    \end{tabular}
\end{table}

We evaluate the performance of KV cache compression methods on \textbf{LongBench}~\cite{bai2024longbench} using the \textbf{Mistral-7B-Instruct-v0.3} model. To ensure a faithful implementation, we align our hyperparameter settings shown in Table~\ref{tab:kv_config} with the notations and definitions proposed in their respective original papers.

\paragraph{StreamingLLM.}
Following the findings of~\cite{xiao_efficient_2024}, we rely on the \textit{Attention Sink} phenomenon to maintain streaming stability. We retain the Key and Value states of the initial tokens as anchors. Specifically, we set the number of sink tokens to 4, combined with a rolling cache of recent tokens. This configuration addresses the "perplexity surge" issue observed when initial tokens are evicted from the window attention~\cite{xiao_efficient_2024}.

\paragraph{SnapKV.}
Consistent with the algorithm described by~\cite{li_snapkv_2024}, we implement the observation-based compression strategy. We define the \textit{Observation Window} size, denoted as $L_{obs}$ in the paper, to be 32. This window captures the attention patterns from the last segment of the prompt. To effectively cluster important features and avoid noise, we utilize the voting mechanism with a max-pooling layer, setting the pooling kernel size to 5.

\paragraph{PyramidKV.}
Building upon the \textit{Pyramidal Information Funneling} hypothesis~\cite{cai_pyramidkv_2024}, we employ a layer-wise dynamic budget allocation. While PyramidKV shares the pooling kernel setting ($k=5$) with SnapKV, we significantly increase the size of the local window, referred to as "instruction tokens" and denoted as $\alpha$ in~\cite{cai_pyramidkv_2024}, to 128. This adjustment of $\alpha$ (where $\alpha > L_{obs}$) ensures that the model can accurately measure the importance of tokens even under high compression rates in upper layers ($k^l$), mitigating the risk of information loss during the pyramidal aggregation process.


\section{Computational Overhead}
\label{sec:computational_overhead}

\begin{table}[h]
  \centering
  \caption{ART incurs negligible overhead when early termination is disabled.
With $p=\infty$, the ART detector executes at every FlashAttention call without triggering termination, increasing kernel time by only 1.3\%.}
  \label{tab:overhead}
  \small
  \begin{tabular}{lc}
    \toprule
    Setting & Avg. FA kernel time (ms)\\
    \midrule
    Baseline  & 0.78 \\ patience=$\infty$ & 0.79 \\
    \bottomrule
  \end{tabular}
\end{table}

This computational overhead isolates the intrinsic runtime overhead of ART’s early-termination detector, independent of any actual early stopping.
Our goal is to verify that the detector itself introduces negligible computational cost.

ART is evaluated after each tile completes the standard streaming softmax update and value aggregation, followed by a block-wide synchronization.
This design ensures deterministic control flow and avoids warp divergence.
Crucially, the detector incurs a constant amount of computation per tile and introduces no additional global memory accesses; therefore, the overall complexity remains proportional to the number of processed KV tiles.




To quantify this overhead, we disable early termination by setting the patience parameter to $p=\infty$, forcing the ART detector to execute at every FlashAttention invocation without ever triggering termination.
This configuration allows us to measure the pure computational overhead of the detection mechanism.
As reported in Table~\ref{tab:overhead}, enabling ART increases the average FlashAttention kernel runtime by only 1.3\% compared to the baseline, while producing identical evaluation scores.
These results confirm that ART introduces negligible computational overhead. 

\section{Parameter Sensitivity}
\label{sec:Parameter Sensitivity}

\begin{figure}[h]
\centering
\includegraphics[width=\linewidth]{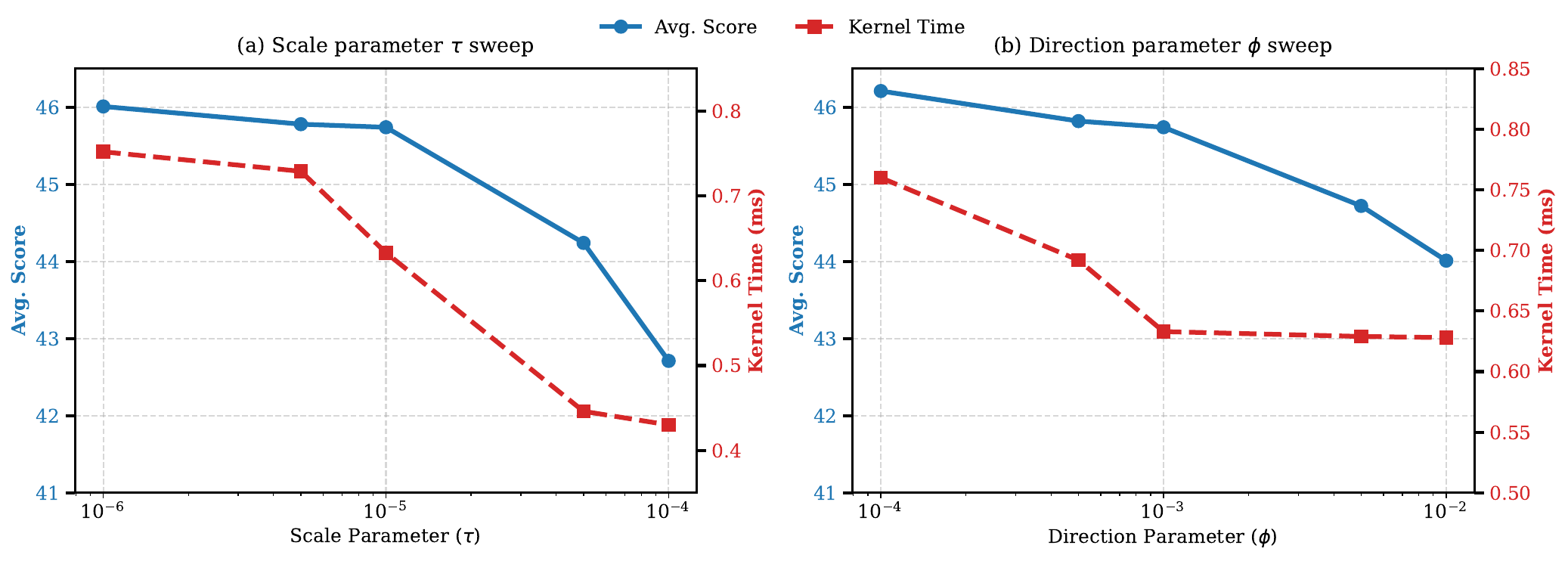}
\caption{Sensitivity analysis of the scale parameter $\tau$. The blue solid line (left axis) tracks the LongBench score, while the red dashed line (right axis) shows the kernel execution time. The results highlight a trade-off: $\tau=10^{-5}$ and $\phi=10^{-3}$ offers an optimal balance between stability and speed, while larger values provide maximum acceleration with a slight performance cost.}
\label{fig:sensitivity_tauphi}
\end{figure}

\begin{figure}[h]
\centering
\includegraphics[width=.6\linewidth]{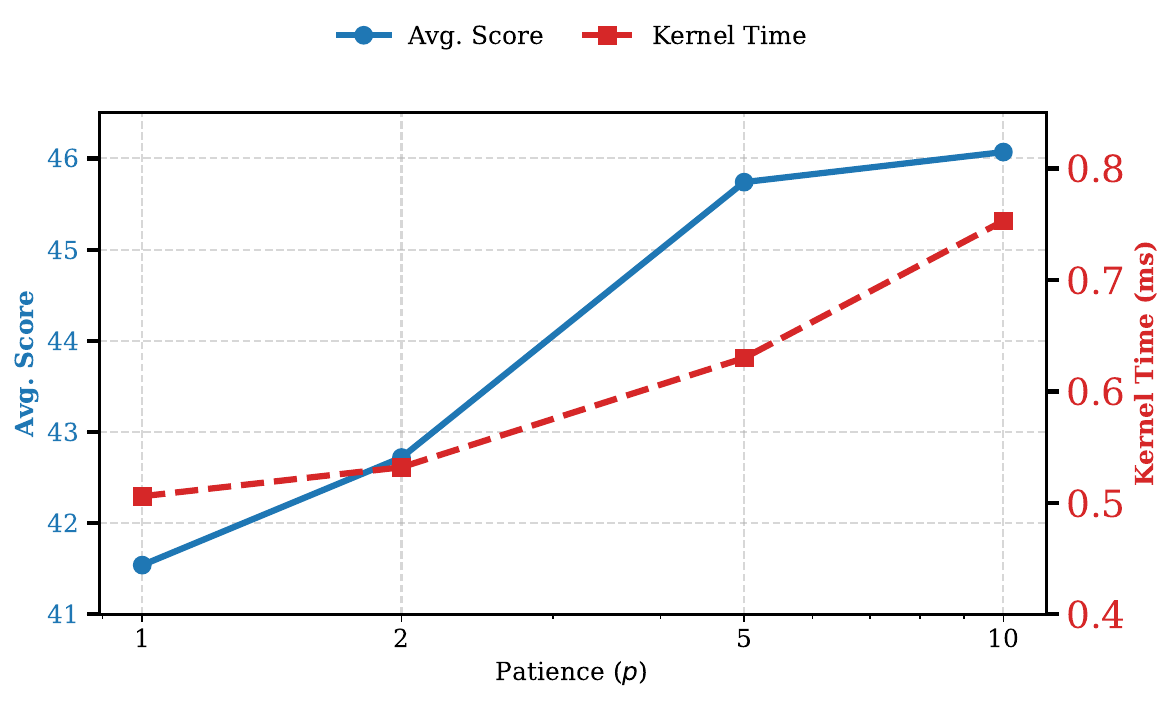}
\caption{Sensitivity analysis of the parameter $p$. The blue solid line (left axis) tracks the LongBench score, while the red dashed line (right axis) shows the kernel execution time.}
\label{fig:sensitivity_p}
\end{figure}

We further investigate the impact of the scale parameter $\tau$ and direction parameter $\phi$ . 
As shown in Figure~\ref{fig:sensitivity_tauphi}, we vary $\tau$ from $10^{-6}$ to $10^{-4}$ and $\phi$ from $10^{-4}$ to $10^{-2}$.
Increasing $\tau$ significantly reduces the kernel execution time, reaching a peak speedup at $10^{-4}$ .
Regarding the quality of the results, the method exhibits strong robustness for smaller thresholds where $\tau \le 10^{-5}$. The average score remains stable around 46.0.
At the most aggressive setting , we observe a moderate performance drop to 44.71.
This creates a flexible trade-off space: users can select $\tau =10^{-5}$ as a balanced operating point to enjoy reduced FA kernel time with negligible accuracy loss, or larger thresholds when speed is the primary priority.

As discussed in our ablation study, $\phi$ primarily serves to preserve correctness. 
The final value of $\phi$ is chosen to capture directional stability under a fixed scale parameter $\tau$. 
When $\tau = 10^{-5}$, the corresponding best-performing value is $\phi = 10^{-3}$. 
A more relaxed, i.e., larger, value of $\phi$ can degrade accuracy, but does not yield a significant efficiency improvement.

We further study the sensitivity of ART to the patience parameter $p$, which controls how many consecutive KV blocks must satisfy the termination criterion before early stopping is triggered. As shown in Figure~\ref{fig:sensitivity_p}, a smaller patience value such as $p=1$ or $p=2$ leads to more aggressive termination and slightly lower kernel time, but it also increases the risk of premature stopping and results in a noticeable drop in average score. Increasing $p$ improves robustness by requiring the attention output to remain stable for multiple consecutive blocks. In particular, $p=5$ achieves the best trade-off between accuracy and efficiency: it recovers most of the performance while retaining a clear kernel-time reduction. Further increasing the patience to $p=10$ only brings marginal accuracy improvement, but substantially weakens the efficiency gain due to later termination. Therefore, we use $p=5$ as the default setting in our experiments, as it provides a balanced operating point between reliable output stability detection and practical acceleration.




\section{Impact of Context Length}
\label{exp:context_length}

To further verify the scalability of our approach, we break down the performance into three context length intervals: 0-4k, 4-8k, and 8k+. As illustrated in Figure \ref{fig:length_breakdown}, the efficiency advantage of ART becomes increasingly pronounced as the input length grows. In short-context scenarios, the runtime difference is marginal; however, in the long-context regime, ART significantly mitigates the computational overhead.

Crucially, this efficiency gain does not compromise long-context capabilities. 
As shown in Figure~\ref{fig:length_breakdown} , the score trajectories of ART-enhanced methods closely track their original counterparts across all length buckets. The overlapping trends indicate that ART robustly preserves important information retrieval abilities even as the context length of the task increases.

\begin{figure*}[h]
  \centering
  \includegraphics[width=\linewidth]{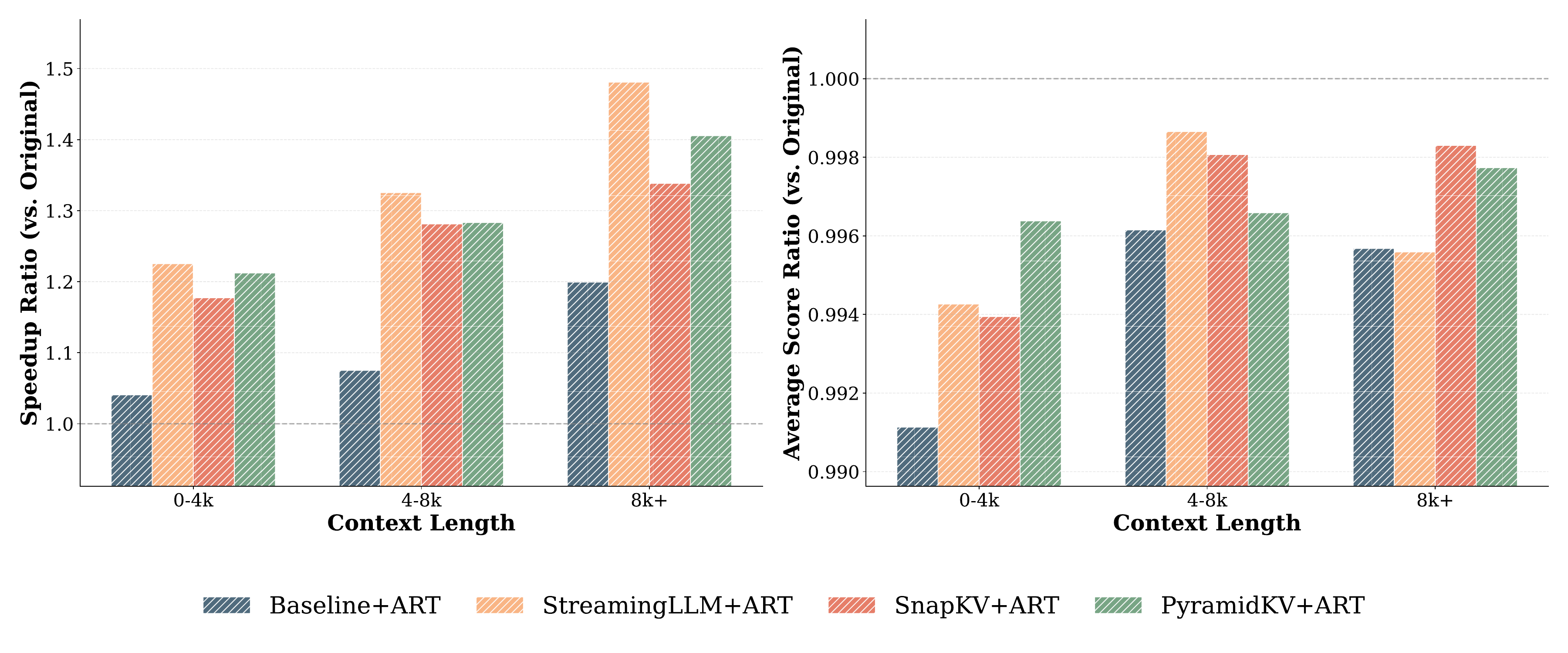} 
  \vspace{-2em}
  \caption{Impact of context length on inference efficiency and performance. The dataset is categorized into three length intervals: 0-4k, 4-8k, and 8k+. The integration of ART significantly reduces kernel running time on long context length with minimal effect on the quality of the results. }
  \label{fig:length_breakdown}
\end{figure*}

\section{Full LongBench Results}
\label{app:longbench_breakdown}

Table~\ref{tab:longbench_scores_appendix} provides the full category-level LongBench breakdown for Mistral-7B-Instruct-v0.3 and Llama-3.1-70B-Instruct. The results are consistent with the average-score summary in the main text: ART introduces only small category-level fluctuations while preserving the overall LongBench score across dense and sparse KV-cache settings.

For the larger-scale evaluation, we deploy Llama-3.1-70B-Instruct on a computing node equipped with four NVIDIA A100 80GB GPUs. This setting evaluates whether ART remains effective under larger model scales and multi-GPU inference. As summarized in Table~\ref{tab:longbench_avg_main}, ART consistently reduces decoding TPOT for Llama-3.1-70B across full KV caching and sparse KV-cache methods, including StreamingLLM, SnapKV, and PyramidKV. The relative speedup is more moderate than that observed on Mistral-7B, which is expected because the overall decoding time composition changes with model scale. As the parameter count increases, the cost of loading weights and computing linear projections, especially in FFN layers, grows substantially. Meanwhile, both Mistral-7B and Llama-3.1-70B use Grouped-Query Attention (GQA), so the KV-cache size does not grow proportionally with the total parameter count. Consequently, attention accounts for a smaller fraction of the total decoding time in the 70B model, limiting the end-to-end speedup from an attention-specific optimization. Nevertheless, ART still reduces the attention-side cost and remains a useful acceleration layer in large-scale long-context inference.

\begin{table*}[t]
\centering
\caption{Category-level LongBench score and decoding TPOT speedup comparison on Mistral-7B-Instruct-v0.3 and Llama-3.1-70B-Instruct.
Each category score is averaged over its corresponding tasks. The \textbf{Avg.} column reports the original LongBench mean score over all 21 tasks, $\Delta$ denotes the marginal score change compared with the corresponding base method.}
\label{tab:longbench_scores_appendix}

\small
\setlength{\tabcolsep}{2.8pt}
\renewcommand{\arraystretch}{0.92}
\begin{tabular}{@{}lccccccccc@{}}
\toprule
\textbf{Method}
& \textbf{Single QA}
& \textbf{Multi QA}
& \textbf{Summ.}
& \textbf{Few-shot}
& \textbf{Synth.}
& \textbf{Code}
& \textbf{Avg.}
& $\boldsymbol{\Delta}$
& \textbf{TPOT speedup} \\
\midrule

\multicolumn{10}{c}{\textbf{Mistral-7B-Instruct-v0.3}} \\
\midrule

\textbf{Baseline}
& 37.88 & 34.65 & 26.06 & 63.43 & 65.63 & 63.56
& \textbf{46.29} & -- & 1.00 $\times$\\

\quad + ART
& 37.20 & 33.96 & 26.40 & 63.04 & 64.85 & 61.80
& \textbf{45.74} & -0.55 & \textbf{1.16$\times$} \\

\midrule

\textbf{StrLLM(0.8)}
& 28.34 & 27.61 & 25.81 & 62.27 & 47.21 & 64.16
& \textbf{40.29} & -- & 1.08$\times$ \\

\quad + ART
& 28.05 & 27.34 & 25.83 & 62.42 & 47.51 & 63.52
& \textbf{40.20} & -0.09 & \textbf{1.24$\times$} \\

\midrule

\textbf{StrLLM(0.2)}
& 20.04 & 19.69 & 22.71 & 56.57 & 18.20 & 61.27
& \textbf{31.11} & -- & 1.18$\times$ \\

\quad + ART
& 19.98 & 19.60 & 22.70 & 56.74 & 18.20 & 61.11
& \textbf{31.09} & -0.02 & \textbf{1.30$\times$} \\

\midrule

\textbf{SnapKV(0.8)}
& 33.13 & 35.56 & 26.33 & 61.36 & 64.67 & 51.33
& \textbf{43.92} & -- & 1.07$\times$ \\

\quad + ART
& 34.05 & 31.76 & 25.39 & 59.62 & 65.11 & 53.45
& \textbf{43.12} & -0.80 & \textbf{1.17$\times$} \\

\midrule

\textbf{SnapKV(0.2)}
& 19.96 & 19.35 & 21.04 & 46.78 & 44.34 & 46.49
& \textbf{31.17} & -- & 1.17$\times$ \\

\quad + ART
& 19.77 & 19.07 & 21.07 & 45.50 & 44.09 & 46.39
& \textbf{30.80} & -0.37 & \textbf{1.27$\times$} \\

\midrule

\textbf{PyramidKV(0.8)}
& 37.02 & 35.08 & 26.04 & 55.48 & 56.25 & 55.37
& \textbf{42.57} & -- & 1.06$\times$ \\

\quad + ART
& 36.28 & 35.02 & 25.92 & 53.96 & 55.21 & 55.77
& \textbf{42.00} & -0.57 & \textbf{1.19$\times$} \\

\midrule

\textbf{PyramidKV(0.2)}
& 23.30 & 30.31 & 22.62 & 47.54 & 12.69 & 49.18
& \textbf{30.07} & -- & 1.17$\times$ \\

\quad + ART
& 23.50 & 29.71 & 22.70 & 47.66 & 12.92 & 48.11
& \textbf{29.97} & -0.10 & \textbf{1.29$\times$} \\

\midrule
\midrule

\multicolumn{10}{c}{\textbf{Llama-3.1-70B-Instruct}} \\
\midrule

\textbf{Baseline}
& 45.47 & 41.48 & 24.36 & 60.17 & 70.31 & 58.20
& \textbf{48.25} & -- & 1.00$\times$ \\

\quad + ART
& 45.99 & 38.42 & 24.41 & 60.27 & 69.53 & 56.05
& \textbf{47.48} & -0.77 & \textbf{1.15$\times$} \\

\midrule

\textbf{StrLLM(0.8)}
& 33.97 & 35.81 & 25.16 & 59.53 & 50.49 & 59.61
& \textbf{42.31} & -- & 1.13$\times$ \\

\quad + ART
& 33.68 & 34.77 & 24.69 & 59.44 & 50.23 & 58.38
& \textbf{41.80} & -0.51 & \textbf{1.27$\times$} \\

\midrule

\textbf{StrLLM(0.2)}
& 20.54 & 32.19 & 23.11 & 58.23 & 16.37 & 63.06
& \textbf{33.88} & -- & 1.52$\times$ \\

\quad + ART
& 20.39 & 31.47 & 24.00 & 56.64 & 16.25 & 62.52
& \textbf{33.51} & -0.37 & \textbf{1.66$\times$} \\

\midrule

\textbf{SnapKV(0.8)}
& 38.83 & 40.48 & 24.41 & 57.10 & 70.41 & 60.93
& \textbf{46.49} & -- & 1.08$\times$ \\

\quad + ART
& 36.36 & 39.77 & 25.72 & 55.83 & 69.19 & 59.09
& \textbf{45.54} & -0.95 & \textbf{1.18$\times$} \\

\midrule

\textbf{SnapKV(0.2)}
& 24.70 & 33.19 & 21.67 & 48.38 & 32.03 & 55.56
& \textbf{34.24} & -- & 1.47$\times$ \\

\quad + ART
& 22.97 & 32.31 & 22.30 & 47.79 & 31.51 & 54.70
& \textbf{33.59} & -0.65 & \textbf{1.58$\times$} \\

\midrule

\textbf{PyramidKV(0.8)}
& 45.27 & 41.93 & 24.87 & 56.45 & 46.76 & 55.80
& \textbf{44.10} & -- & 1.12$\times$ \\

\quad + ART
& 44.42 & 40.58 & 24.50 & 54.20 & 48.06 & 55.27
& \textbf{43.31} & -0.79 & \textbf{1.25$\times$} \\

\midrule

\textbf{PyramidKV(0.2)}
& 28.64 & 44.07 & 22.65 & 49.72 & 28.91 & 43.74
& \textbf{35.93} & -- & 1.47$\times$ \\

\quad + ART
& 27.76 & 43.70 & 22.41 & 47.92 & 30.47 & 43.17
& \textbf{35.47} & -0.46 & \textbf{1.60$\times$} \\

\bottomrule
\end{tabular}
\end{table*}

\section*{Broader Impacts}
ART aims to improve the efficiency of long-context LLM inference by reducing redundant
attention computation and KV-cache accesses. Potential positive impacts include lower
serving cost, reduced energy consumption, and improved accessibility of long-context
applications under limited hardware resources. At the same time, more efficient inference
can lower the cost of deploying large-scale generative systems, which may indirectly
amplify misuse risks already associated with LLMs, such as spam, misinformation, or
automated content generation. This work does not introduce a new generative model,
training dataset, or user-facing deployment system; the released assets are limited to
efficiency-oriented implementation and evaluation code. We encourage users to apply ART
within deployments that follow the safety policies and usage restrictions of the underlying
models.